\documentclass{article}

\usepackage[preprint]{neurips_2026}

\usepackage[utf8]{inputenc}
\usepackage[T1]{fontenc}
\usepackage{microtype}
\usepackage{graphicx}
\usepackage{booktabs}
\usepackage{comment}
\usepackage{hyperref}
\usepackage{url}
\usepackage{verbatim}
\usepackage{amsmath,amssymb,amsfonts,amsthm}
\usepackage{nicefrac}
\usepackage{bm}
\usepackage{mathtools}
\usepackage{enumitem}
\usepackage{algorithm}
\usepackage{algorithmic}
\usepackage{xcolor}

\newtheorem{theorem}{Theorem}[section]
\newtheorem{lemma}[theorem]{Lemma}

\title{A Physical Theory of Backpropagation: Exact Gradients from the Least-Action Principle}

\author{%
  Antonino Emanuele Scurria \\
  Quantum Information Laboratory (LIQ) CP224\\
  Universit\'e libre de Bruxelles (ULB)\\
  Av.\ F.\ D.\ Roosevelt 50, 1050 Bruxelles, Belgium \\
  \texttt{antonino.scurria@ulb.be} \\
}

\begin{document}

\maketitle

\begin{abstract}
Backpropagation is typically presented as a symbolic procedure: a backward pass topologically distinct from inference, with non-local error signals and synchronous global clocking, features with no clear analog in physical reality. Existing physics-inspired alternatives recover gradients only approximately, in vanishing-perturbation limits, or under weight-symmetry constraints incompatible with feedforward architectures.
In this paper, we address this gap by deriving exact backpropagation from Hamilton's least-action principle. By recasting the forward dynamics in continuous time and adapting a Lagrangian formalism for non-conservative systems to the resulting flow, we unify inference and gradient computation within a single variational framework on a doubled phase space, whose two conjugate fields jointly encode activations and sensitivities. A single global Lagrangian governs the dynamics: the task loss enters as a symmetry-breaking perturbation of the forward manifold, and credit assignment emerges as the tension that develops between the conjugate states. Inference and gradient computation thus unfold simultaneously through local interactions, requiring no separate backward circuit. Ultimately, standard backpropagation is recovered exactly as the discrete-time projection of this continuous flow.
This perspective unifies the formalism of physics with backpropagation, opening a principled pathway for applying tools from classical mechanics — symplectic geometry, Noether's theorem, path-integral methods — to the analysis of learning dynamics. As a downstream consequence, it also points toward analog and neuromorphic substrates in which learning is embodied in the hardware itself.
\end{abstract}

\section{Introduction}

For nearly four decades, backpropagation has been the engine of deep learning, yet it has resisted a physical derivation. As currently formulated, it is a symbolic procedure: a backward pass topologically distinct from inference, transmitting non-local error signals under synchronous global clocking \citep{crick1989recent}. None of these features has a clear analog in physical reality, and this gap has motivated a long line of attempts to recover gradient computation from physically plausible dynamics.

Each attempt has fallen short of a least-action derivation in a structurally identifiable way. Equilibrium Propagation \citep{SB17} and Contrastive Hebbian Learning \citep{movellan1991contrastive} cast learning as energy minimization on a scalar potential, but recover gradients only in a vanishing-nudge limit and require symmetric weights---a constraint fundamentally incompatible with feedforward Jacobians, which are non-reciprocal. Continuous adjoint methods \citep{chen2018neural} integrate dynamics backward in time, a procedure physically realizable only for time-reversible systems. Recurrent Backpropagation \citep{almeida1990learning, pineda1987generalization} introduces a separate auxiliary linear system engineered specifically to compute gradients. LeCun's Lagrangian formulation \citep{lecun1988theoretical} is, by his own explicit acknowledgment, not a Lagrangian of physics: it is a static algebraic stationarity condition with no associated dynamics, in which the adjoint variables are solved for rather than evolved. Across these approaches, a single tension recurs: physical realism and exact credit assignment have appeared mutually exclusive.

\textbf{In this work we address this tension.} We show that backpropagation is the equilibrium of a least-action flow on a doubled phase space, and that standard backpropagation is recovered exactly under the unit-step Euler discretization of this flow. Specializing the Bateman--Galley formalism for non-conservative classical mechanics \citep{bateman1931dissipative, galley2013classical, aykroyd2025hamiltonian} to the feedforward network flow, we construct a single global Lagrangian whose two conjugate fields jointly encode activations and sensitivities. The Euler--Lagrange equations of this Lagrangian, with the standard Rayleigh dissipation, generate a continuous flow whose unique stable equilibrium encodes the exact backpropagation gradient as the stress between the conjugate states. The task loss enters the action as a finite symmetry-breaking perturbation---no vanishing-nudge limit, no weight symmetry, no reverse-time integration, no engineered backward circuit. Inference and credit assignment unfold simultaneously through local interactions, as two facets of one variational dynamics.

The significance of this reformulation is structural rather than algorithmic. Backpropagation as currently understood is procedurally specified---a symbolic chain-rule computation that physical substrates can only approximate. Backpropagation as a least-action flow is a different kind of object: it is governed by the same variational principle as the rest of classical mechanics, and the full analytical machinery built around least-action principles---Noether's theorem, symplectic geometry, Hamilton--Jacobi theory, path-integral extensions to stochastic and quantum settings---becomes natively applicable to learning dynamics. As a downstream consequence, learning becomes embeddable in physical substrates that \emph{are} least-action systems by construction, rather than requiring digital approximation of a symbolic procedure.

\paragraph{Contributions.}
\begin{itemize}[leftmargin=*, noitemsep, topsep=2pt]
    \item \textbf{Variational derivation.} The first derivation of exact 
    backpropagation from Hamilton's least-action principle, obtained by 
    specializing the Bateman--Galley formalism to the feedforward flow. 
    The construction extends to the full second-order action 
    (Appendix~\ref{app:mass}).
    \item \textbf{Exact gradient, exact discrete recovery.} The unique stable 
    equilibrium encodes the exact BP gradient as the stress between the two 
    conjugate fields; discretization recovers standard BP  digital algorithm. Empirically, on CIFAR-10 \citep{krizhevsky2009cifar} (released by the authors for non-commercial research use) with a 9-layer VGG, gradients match 
    BP to float32 precision ($\sim 10^{-7}$) at parity accuracy ($\sim 93\%$).
    \item \textbf{Physics toolkit becomes native.} Casting learning as a
least-action flow brings the analytical machinery of classical mechanics
---Noether's theorem, symplectic geometry, Hamilton--Jacobi theory,
path-integral methods---natively to bear on learning dynamics; we leave
the systematic exploitation of these tools to future work.
\end{itemize}
The remainder of the paper is organized as follows. Section~\ref{sec:preliminaries} fixes notation and reviews classical backpropagation. Section~\ref{sec:global} lifts the layer-wise dynamics to a global flow and constructs the variational framework, deriving the doubled-phase-space dynamics and analyzing its equilibria. Section~\ref{sec:discrete} proves exact recovery of standard BP. Section~\ref{sec:experiments} reports empirical validation. Section~\ref{sec:related} situates the contribution against prior variational and energy-based approaches.

\section{Classical Backpropagation}
\label{sec:preliminaries}
We fix notation. Consider a standard $L$-layer feedforward network with input
$\bm{x}_0 \in \mathbb{R}^{n_0}$ and output $\bm{a}_L \in \mathbb{R}^{n_L}$.
For each layer $\ell = 1, \dots, L$, the pre-activation
$\bm{z}_\ell \in \mathbb{R}^{n_\ell}$ and activation
$\bm{a}_\ell \in \mathbb{R}^{n_\ell}$ are given by
\begin{equation}
\bm{z}_\ell = \bm{W}_\ell \bm{a}_{\ell-1} + \bm{b}_\ell,
\qquad
\bm{a}_\ell = \sigma_\ell(\bm{z}_\ell),
\end{equation}
where $\bm{W}_\ell \in \mathbb{R}^{n_\ell \times n_{\ell-1}}$,
$\bm{b}_\ell \in \mathbb{R}^{n_\ell}$, and $\sigma_\ell$ is an elementwise
nonlinearity; for $\ell = L$, $\bm{a}_L$ is the network output. We collect
all parameters into
$\bm{\theta} = \{\bm{W}_\ell, \bm{b}_\ell\}_{\ell=1}^{L}$, and denote the
supervised loss by $C(\bm{a}_L, \bm{y})$. Backpropagation computes gradients
via the recursion
\[
\bm{\delta}_\ell = (\bm{W}_{\ell+1}^\top \bm{\delta}_{\ell+1}) \odot 
\sigma'_\ell(\bm{z}_\ell), \qquad
\bm{\delta}_L = \nabla_{\bm{a}_L} C \odot \sigma_L'(\bm{z}_L),
\]
yielding $\partial C/\partial\bm{W}_\ell = \bm{\delta}_\ell \bm{a}_{\ell-1}^\top$ 
and $\partial C/\partial\bm{b}_\ell = \bm{\delta}_\ell$. The layer-wise 
structure induces a strict backward ordering, which we reinterpret as a 
global dynamical system in what follows.

\section{From Layers to Global Flow}
\label{sec:global}

We now move from the layer-wise discrete dynamics of classical backpropagation
(Section~\ref{sec:preliminaries}) to a \emph{global} description.
The key step is to reinterpret the layer index as a discrete time coordinate
and then embed this into a continuous-time vector field defined on a single
stacked state.
This will allow us to apply the Lagrangian theory of non-conservative systems
to the network as a whole in Section~\ref{sec:constr_ener}.

\subsection{Global Activation Vector and Weight Matrix}
We stack all layer activations and bias/input drives into single global
vectors, and collect the inter-layer weights into a single global matrix:
\begin{equation}
\bm{a} =
\begin{bmatrix}
\bm{a}_1 \\ \bm{a}_2 \\ \vdots \\ \bm{a}_L
\end{bmatrix},
\quad
\bm{\beta}(\bm{x}_0) =
\begin{bmatrix}
\bm{W}_1 \bm{x}_0 + \bm{b}_1 \\
\bm{b}_2 \\
\vdots \\
\bm{b}_L
\end{bmatrix},
\quad
\bm{W} =
\begin{bmatrix}
\bm{0}   & \bm{0}   & \cdots & \bm{0}   \\
\bm{W}_2 & \bm{0}   & \cdots & \bm{0}   \\
\vdots   & \ddots   & \ddots & \vdots   \\
\bm{0}   & \cdots   & \bm{W}_L & \bm{0}
\end{bmatrix},
\label{eq:global_W}
\end{equation}
with $\bm{a}, \bm{\beta} \in \mathbb{R}^{n}$ and $n = \sum_{\ell=1}^{L} n_\ell$.
The matrix $\bm{W}$ is strictly lower block-triangular: information flows only
from layer $\ell-1$ to layer $\ell$. The forward pass over all layers can then
be written compactly as the unique fixed point of
\begin{equation}
\bm{a} = \bm{\sigma}(\bm{W}\bm{a} + \bm{\beta}(\bm{x}_0)),
\label{eq:global_fixed_point_prelims}
\end{equation}
where $\bm{\sigma}$ denotes the stacked nonlinearity acting elementwise.

\subsection{The Forward Vector Field}
\label{subsec: forward_vf}
In preparation for the physical formulation, we then proceed to generalize the
feedforward dynamics from discrete to continuous time as the fixed point of a global relaxation dynamics.
We introduce a time-dependent global activation state $\bm{a}(t)$ evolving under
the vector field
\begin{equation}
\frac{d\bm{a}(t)}{dt}
=
\bm{\sigma}\!\big(\bm{W}\bm{a}(t) + \bm{\beta}(\bm{x}_0)\big)
- \bm{a}(t)
=: F(\bm{a}(t)).
\label{eq:global_ct_dynamics}
\end{equation}
By construction, the equilibria of~\eqref{eq:global_ct_dynamics} are exactly the
solutions of the fixed-point equation~\eqref{eq:global_fixed_point_prelims}.
In particular, the standard forward pass configuration corresponds to the unique
stable equilibrium $\bar{\bm{a}}$ satisfying
$\bar{\bm{a}} = \bm{\sigma}(\bm{W}\bar{\bm{a}} + \bm{\beta}(\bm{x}_0))$.

\subsection{Nilpotency}

It is useful to notice that in this global framework, the global connection matrix $\bm{W}$ is strictly block lower-triangular and this implies a strong algebraic property.

\begin{lemma}[Nilpotency of the Global Weight Matrix]
\label{lem:nilpotency}
For an $L$-layer feedforward network, the global matrix $\bm{W}$ satisfies:
\[
\bm{W}^L = \bm{0}.
\]
\end{lemma}

\begin{proof}
Because $\bm{W}$ has nonzero blocks only immediately below the diagonal, each multiplication by $\bm{W}$ shifts nonzero contributions at most one block row down.
After $L$ such shifts, all contributions have moved past the last block row and vanish, hence $\bm{W}^L = \bm{0}$.
\end{proof}

Nilpotency captures the acyclic structure of feedforward networks: information propagates strictly layer by layer with no feedback cycles.
This property will imply finite-time convergence for the relaxation dynamics constructed in later sections.

\subsection{The Least Action path}
\label{sec:constr_ener}

The first step is to build a Lagrangian for our formulation. To do this we first have to address the \emph{non-reciprocal} nature of $F$: information flows from layer $\ell$ to $\ell+1$ without symmetric feedback. This violation of Newton's third law implies that the network vector field $F$ previously defined in Eq. \ref{eq:global_ct_dynamics} is \emph{non-conservative} and cannot be derived from a standard scalar potential \cite{goldstein2002classical}.

To resolve this, we apply the Lagrangian theory of non-conservative
systems. While inspired by the work of \cite{bateman1931dissipative},
\cite{keldysh1964diagram}, and \cite{galley2013classical}, we modify
the formalism to enable learning through equilibria including
dissipation (see Appendix~\ref{app:mechanics_overview}). Crucially,
to variationally encode non-reciprocal interactions, this theory
entails a doubling of the phase space (see also
\citet{Scurria2026Dyadic} for a related doubled-state
construction in the equilibrium-propagation setting). We therefore
introduce a conjugate pair of global state vectors: a ``backward''
state $\bm{x}\in\mathbb{R}^n$ and a ``forward'' state
$\bm{z}\in\mathbb{R}^n$.

Starting from the definition of the force $F$, we can thus build the
Lagrangian of our system (see Appendix \ref{app:lagr_build} for a
detailed explanation):
\vspace{-0.5em}
\begin{equation}
\mathcal{L}(\bm{x},\bm{z})
=
\underbrace{(\bm{x}-\bm{z})^{\!\top}
F\!\left(\tfrac{\bm{x}+\bm{z}}{2}\right)}_{\text{Lagrangian Lifting}}
+
\underbrace{C\!\left(
\Big[\tfrac{\bm{x}+\bm{z}}{2}\Big]_L,
\bm{y}
\right)}_{\text{Cost}}.
\label{eq:energy_xz_icml}
\end{equation}
Notice that in this Lagrangian we do not have a mass term (usually termed 'overdamped regime') to adapt to the dynamics of the previous section in which we do not have an acceleration term; this can be however included obtain the same equilibria and stability properties (see Appendix \ref{app:mass}). The inclusion of the cost directly in the Lagrangian breaks the symmetry of the two states $\bm{x}$ and $\bm{z}$ with a perturbation that will encode the gradient (see Appendix \ref{app:lagr_build}). 

The equations of motion for the doubled system are then obtained through the least action principle:
\begin{equation}
    \delta\left(\int_{t_1}^{t_2}\mathcal{L}(x,\dot{x},z,\dot{z})dt\right)=0
\end{equation}
\vspace{-0.5em}
This gives us the equations
\begin{equation}
\begin{aligned}
\frac{d}{dt}\left(\frac{\partial \mathcal{L}}{\partial \dot{x}}\right) - \frac{\partial \mathcal{L}}{\partial x} &= -\frac{\partial R}{\partial \dot{x}} \implies  \dot{x} = +\frac{\partial \mathcal{L}}{\partial x} \\
-\left(\frac{d}{dt}\left(\frac{\partial \mathcal{L}}{\partial \dot{z}}\right) - \frac{\partial \mathcal{L}}{\partial z}\right) &= -\frac{\partial R}{\partial \dot{z}} \implies \dot{z} = -\frac{\partial \mathcal{L}}{\partial z}
    \end{aligned}
    \end{equation}
\vspace{-0.5em}
\noindent 
where $R = \frac{1}{2}(\dot{x}^2 + \dot{z}^2)$ is the so called Rayleigh dissipation function. Thus we have:
\begin{equation}
\begin{aligned}
\frac{d\bm{x}}{dt}
&=
\underbrace{F\!\left(\tfrac{\bm{x}+\bm{z}}{2}\right)}_{\text{Network Relaxation}}
+ \underbrace{\tfrac{1}{2}\,[\nabla F(\tfrac{\bm{x}+\bm{z}}{2})]^{\!\top}\,(\bm{x}-\bm{z})}_{\text{Backward Signal}}
+ \underbrace{\tfrac{1}{2}
\begin{bmatrix} \bm{0} \\ \vdots \\ \nabla_{\bm{a}_L} C\!\big([\tfrac{\bm{x}+\bm{z}}{2}]_L,\bm{y}\big) \end{bmatrix}}_{\text{Cost}},
\\
\frac{d\bm{z}}{dt}
&=
\underbrace{F\!\left(\tfrac{\bm{x}+\bm{z}}{2}\right)}_{\text{Network Relaxation}}
- \underbrace{\tfrac{1}{2}\,[\nabla F(\tfrac{\bm{x}+\bm{z}}{2})]^{\!\top}\,(\bm{x}-\bm{z})}_{\text{Backward Signal}}
- \underbrace{\tfrac{1}{2}
\begin{bmatrix} \bm{0} \\ \vdots \\ \nabla_{\bm{a}_L} C\!\big([\tfrac{\bm{x}+\bm{z}}{2}]_L,\bm{y}\big) \end{bmatrix}}_{\text{Cost}}.
\label{eq:xz_saddle_dynamics}
\end{aligned}
\end{equation}

These dynamics consist of three distinct physical forces. \textbf{Network Relaxation} (first term in both equations) is identical for both states, driving $\bm{x}$ and $\bm{z}$ toward the forward pass configuration. \textbf{Backward Signal} uses the stress $\bm{x-z}$, to propagate error information. \textbf{Cost} acts as an external force applied only to the output block.

Crucially, this formalism requires identical initial conditions for both x and z (also called "the physical limit", see \cite{galley2013classical}). In this scenario, absent the cost term, both states collapse to the same dynamics, corresponding to the forward pass dynamics previously defined in section \ref{subsec: forward_vf}. Thus the cost is perturbing the system in order to encode the gradients in the stress term $(\bm{x}-\bm{z})$ (see Appendix \ref{app:symmetry_breaking}).

The Euler discretization of the flow~\eqref{eq:xz_saddle_dynamics} yields a joint update in which both states evolve on the same lattice at every step providing, as we will prove later, exact gradients with only local computations in both space and time; we refer to this scheme as Dyadic Backpropagation (Algorithm~\ref{alg:xz_schematic}, deferred to Appendix~\ref{app:dbp_alg}).
A simplified dynamical scheme (Appendix~\ref{app:split_dynamics}), tested with identical accuracy, is also provided to facilitate hardware implementation.

\subsection{Mean-Stress Variables}
\label{sec:m-s}
To better inspect the interaction between $\bm{x}$ and $\bm{z}$, we work with the mean--stress variables $\bm{m}=\tfrac12(\bm{x}+\bm{z})$ and $\bm{s}=\bm{x}-\bm{z}$ already introduced above. Differentiating these relations yields the transformed system:
\begin{equation}
\begin{aligned}
\label{eq:ms_dynamics}
\frac{d\bm{m}}{dt}
&=
\bm{\sigma}(\bm{W}\bm{m} + \bm{\beta}(\bm{x}_0)) - \bm{m},
\\
\frac{d\bm{s}}{dt}
&=
\big(\bm{W}^\top \bm{D}(\bm{m}) - \bm{I}\big)\bm{s}
+ \nabla_{\bm{m}_L}C(\bm{m}_L,\bm{y}),
\end{aligned}
\end{equation}
where $\bm{D}(\bm{m}) = \operatorname{diag}\!\big(\bm{\sigma}'(\bm{W}\bm{m} + \bm{\beta}(\bm{x}_0))\big)$ is the diagonal matrix of local activation derivatives evaluated at the current mean state.

In this representation, the physical roles of the two variables become
clear: $\bm{m}$ relaxes toward a consistent forward activation pattern, while
$\bm{s}$ records the tension that enforces this consistency and channels
sensitivity information backward through the network.
% It is also instructive to notice that, after the relaxation of $\bm{m}$, the Lagrangian of the system equals the task loss (see Appendix \ref{app:eq_en}).

\subsection{The Equilibrium States}
We now inspect the equilibria of the system in the $\bm{m}$, $\bm{s}$ variables to
better understand the configuration to which the system converges. The equilibria
$\bar{\bm{m}}$ and $\bar{\bm{s}}$ satisfy
\begin{equation}
\bm{\sigma}(\bm{W}\bar{\bm{m}} + \bm{\beta}(\bm{x}_0)) = \bar{\bm{m}},
\qquad
(\bm{I}-\bm{W}^\top \bar{\bm{D}})\,\bar{\bm{s}} = \nabla_{\bar{\bm{m}}_L}C(\bar{\bm{m}}_L,\bm{y}).
\label{eq:s_equilibrium_linear_full}
\end{equation}
The first equation enforces relaxation to the forward pass configuration. The
second yields exactly the same result as backpropagation (see
Appendix~\ref{subsec:bp_eq}): $\bar{\bm{s}}$ collects the total derivatives of
the cost with respect to the equilibrium activations,
$\bar{\bm{s}}_\ell = \nabla_{\bar{\bm{m}}_\ell} C$. Stacking everything into a
single vector and applying the diagonal derivative $\bar{\bm{D}}$ recovers the
classical pre-activation errors:
\begin{equation}
\bar{\bm{s}}=
\begin{bmatrix}
\nabla_{\bar{\bm{m}}_1} C \\
\nabla_{\bar{\bm{m}}_2} C \\
\vdots \\
\nabla_{\bar{\bm{m}}_L} C
\end{bmatrix},
\qquad
\bm{\Delta}=\bar{\bm{D}}\,\bar{\bm{s}}
=
\begin{bmatrix}
\bm{\delta}_1 \\
\bm{\delta}_2 \\
\vdots \\
\bm{\delta}_L
\end{bmatrix},
\qquad
\bm{\delta}_\ell=D_\ell(\bar{\bm{m}}_\ell)\,\bar{\bm{s}}_\ell,
\label{eq:backprop_s}
\end{equation}
where $D_\ell(\bar{\bm{m}}_\ell)$ is the diagonal derivative matrix at layer $\ell$.
Finally, this equilibrium configuration $(\bar{\bm{m}},\bar{\bm{s}})$ is globally
exponentially stable (see Appendix~\ref{subsec:stability}).

\subsection{Gradient Computation}
\label{sec:grad_comp}
Now we use the equilibrium configuration of our system to compute the total derivative of the cost with respect to the weights, denoted $\nabla_{\bm{W}} C$. Using the chain rule, the total gradient can be expressed in terms of the equilibrium variables:
\begin{equation}
\nabla_{\bm{W}} C
=\sum_{j} \nabla_{\bar{\bm{m}}_{j}} C \frac{\partial \bar{\bm{m}}_{j}}{\partial \bm{W}}=
\sum_{j} \bar{s}_j \frac{\partial (\bm{\sigma}(\bm{W}\bar{\bm{m}} + \bm{\beta}))_j}{\partial \bm{W}}.
\label{eq:chain_rule_adjoint}
\end{equation}

Evaluating the partial derivative of the update rule $\bar{m}_i = \sigma_i(\sum_j W_{ij}\bar{m}_j + \beta_i)$ yields:
\begin{equation}
\frac{\partial \bar{m}_i}{\partial W_{ij}}
=
D_{ii}(\bar{\bm{m}})\, \bar{m}_j,
\end{equation}
where $D_{ii}(\bar{\bm{m}})$ denotes the $i$-th block of $\bm{D}(\bar{\bm{m}})$.
Substituting this into~\eqref{eq:chain_rule_adjoint}, we obtain the standard gradient outer product:
\begin{equation}
\nabla_{\bm{W}} C
=
\bm{\Delta}\,\bar{\bm{m}}^{\!\top},
\qquad
\bm{\Delta}=\bm{D}(\bar{\bm{m}})\,\bar{\bm{s}}.
\label{eq:global_chain_rule_explicit_full}
\end{equation}
In layer form, this recovers the familiar local gradients:
\begin{align}
\nabla_{\bm{W}_\ell} C
&=
\bm{\delta}_\ell\,\bar{\bm{m}}_{\ell-1}^{\!\top},
\\
\nabla_{\bm{b}_\ell} C
&=
\bm{\delta}_\ell.
\end{align}
These expressions highlight the two-factor structure of gradient computation: a pre-synaptic activation $\bar{\bm{m}}_{\ell-1}$ and a post-synaptic variable $\bm{\delta}_\ell$.

% -----------------------------------------------------
% -----------------------------------------------------
% -----------------------------------------------------
% -----------------------------------------------------
% -----------------------------------------------------
% -----------------------------------------------------
\section{The Discrete Trace: Recovering Backpropagation}
\label{sec:discrete}

We now study the discretized version of the continuous-time $(\bm{m},\bm{s})$ dynamics Eq.~\eqref{eq:ms_dynamics} to explicitly show equivalence to the backpropagation dynamics.

\subsection{Canonical Discretization and the Natural Time Step}

Applying forward Euler with step size $\Delta t>0$ to the equation \ref{eq:ms_dynamics} gives the discrete updates:
\begin{align}
\bm{m}^{k+1}
&=
\bm{m}^k
+
\Delta t\left[
\bm{\sigma}(\bm{W}\bm{m}^k + \bm{\beta}(\bm{x}_0))
- \bm{m}^k
\right],
\label{eq:euler_m_general_full}
\\
\bm{s}^{k+1}
&=
\bm{s}^k
+
\Delta t\left[
\big(\bm{W}^\top \bm{D}(\bm{m}^k) - \bm{I}\big)\bm{s}^k
+ \nabla_{\bm{m}_L}C(\bm{m}_L^k,\bm{y})
\right].
\label{eq:euler_s_general_full}
\end{align}
Standard layer-by-layer backpropagation corresponds to choosing a \emph{unit} Euler step, $\Delta t = 1$.
This specific choice exactly cancels the inertial term $(1-\Delta t)\bm{m}^k$.
In that case, the equations simplify to the coupled map:
\begin{align}
\bm{m}^{k+1}
&=
\bm{\sigma}(\bm{W}\bm{m}^k + \bm{\beta}(\bm{x}_0)),
\label{eq:m_discrete_2}
\\
\bm{s}^{k+1}
&=
\bm{W}^\top \bm{D}(\bm{m}^k)\bm{s}^k
+
\nabla_{\bm{m}_L}C(\bm{m}_L^k,\bm{y}).
\label{eq:s_discrete_2}
\end{align}

\subsection{Nilpotency and Exact Convergence in 2L Steps}

 We now show that we obtain finite time convergence of the discretized dynamical equations (\ref{eq:m_discrete_2})(\ref{eq:s_discrete_2}) in exactly $2L$ steps, where $L$ corresponds to the depth of the network.
This can be easily seen because $\bm{W}$ is strictly lower block-triangular. Then each block $\bm{m}_\ell$ depends only on upstream layers $\bm{m}_1,\dots,\bm{m}_{\ell-1}$, ensuring strictly sequential convergence.

\begin{lemma}[Forward Layer Freezing]
For the iteration~\eqref{eq:m_discrete_2} with initial condition $\bm{m}^0$, the $\ell$-th block satisfies
\[
\bm{m}_\ell^k = \bm{a}_\ell
\quad\text{for all}\quad k\ge \ell,
\]
where $\bm{a}_\ell$ denotes the standard forward activation at layer $\ell$.
\end{lemma}

\begin{proof}
By induction on $\ell$, using the strict block lower-triangularity of $\bm{W}$. See Appendix~\ref{app:forward_freezing_proof} for the full argument.
\end{proof}

Let $\bar{\bm{m}}$ denote the stacked forward activations at equilibrium. It follows that for all $k\ge L$, $\bm{m}^k=\bar{\bm{m}}$. Thus the forward relaxation converges in exactly $L$ iterations.

We now need to inspect the convergence of the $\bm{s}$ variable.

Crucially, the $\bm{s}$ update in Eq.~\eqref{eq:s_discrete_2} uses the matrix $\bm{D}(\bm{m}^k)$.
We can then use the fact that for $k\geq L$ we get
\[
\bm{D}(\bm{m}^k) = \bm{D}(\bar{\bm{m}}) =: \bar{\bm{D}}, \quad \text{for all } k \geq L.
\]
Consequently, for all subsequent steps $k \ge L$, the update for $\bm{s}$ becomes a linear dynamical system with a \emph{constant} transition matrix $\bar{\bm{M}} = \bm{W}^\top \bar{\bm{D}}$:
\[
\bm{s}^{k+1} = \bar{\bm{M}} \bm{s}^k + \nabla C, \quad \text{for } k \ge L.
\]
The convergence of this phase relies on the algebraic properties of $\bar{\bm{M}}$. Because the network is feedforward, $\bar{\bm{M}}$ is strictly upper block-triangular and thus \emph{nilpotent} with index $L$ (i.e., $\bar{\bm{M}}^L = \bm{0}$).

This nilpotency guarantees that any transient (or residual) information accumulated in $\bm{s}$ during the initial transient phase (steps $0$ to $L-1$, when $\bm{D}$ is changing) is progressively shifted out of the network. After exactly $L$ applications of the constant operator $\bar{\bm{M}}$ (i.e., at step $k = L + L = 2L$), the memory of the initial transient is completely annihilated, and the state settles exactly into the fixed point previously discussed (see Eq.~\ref{eq:backprop_s}).

More formally:

\begin{theorem}[Discrete Equivalence to Standard Backpropagation]
\label{thm:autonomous_convergence}
Let the coupled system evolve simultaneously through Eqs.~\ref{eq:m_discrete_2}~and~\ref{eq:s_discrete_2}. Then:
\begin{enumerate}
    \item \textbf{Forward Settling ($k < L$):} The primal state converges to the valid network activations at $k=L$, i.e., $\bm{m}^k = \bar{\bm{m}}$ for all $k \geq L$.
    \item \textbf{Backward Flushing ($L \leq k < 2L$):} With $\bm{m}$ fixed, the dual state flushes out transient errors and converges to the exact backpropagation gradients at $k=2L$, i.e., $\bm{s}^k = \bar{\bm{s}}$ for all $k \geq 2L$.
\end{enumerate}
The system thus reproduces the backpropagation gradient autonomously in $2L$ steps, without external mode switching.
\end{theorem}

The corresponding pseudocode (Algorithm~\ref{alg:autonomous_bp}) is deferred to Appendix~\ref{app:autonomous_alg}.

Together, Theorem~\ref{thm:autonomous_convergence} and Algorithm~\ref{alg:autonomous_bp} establish that standard backpropagation is the exact discrete trace of the continuous $\bm{x}$--$\bm{z}$ dynamics in Eq. \ref{eq:xz_saddle_dynamics}. The conventional two-phase procedure of the forward and backward passes is thus revealed to be a computational optimization of this single physical process, skipping redundant updates during the system's natural settling phases (we stop updating $\bm{m}$ as soon as it settles to its equilibrium). Consequently, the familiar forward and backward passes can be framed not as arbitrary engineering choices, but as an optimized digital implementation of the intrinsic relaxation behaviors of the underlying physical system.

\section{Empirical Validation}
\label{sec:experiments}
We empirically validate the $(\bm{x},\bm{z})$ dynamics
Eq.~\eqref{eq:xz_saddle_dynamics} on CIFAR-10, comparing performance and
internal dynamics against standard backpropagation (BP). The dynamics are
integrated via forward Euler as detailed in Algorithm~\ref{alg:xz_schematic}
(Appendix~\ref{app:dbp_alg}).

\subsection{Setup}
We use a 9-layer VGG-style CNN ($\sim$5M parameters) with four convolutional blocks (Conv-ReLU-Conv-ReLU-Stride2Pool) of channel dimensions $[64, 128, 256, 512]$ and a linear classifier. Models are trained for 100 epochs with batch size 64 using SGD with Nesterov momentum ($\mu=0.9$), weight decay ($5\times 10^{-4}$), cosine learning-rate annealing ($0.035 \to 0.0002$), standard augmentation (random crops, horizontal flips, Cutout), and label smoothing ($\epsilon=0.1$). The Euler step size is varied over $\eta \in \{0.25, 0.50, 0.75, 1.00\}$ with up to $K_{\max}=1000$ relaxation iterations per training step; early stopping triggers when consecutive iterates differ by less than $10^{-6}$ in Euclidean norm, and all runs share random seeds. Gradient fidelity is measured against PyTorch~2.7.1 autodiff \citep{paszke2019pytorch} (BSD-3-Clause license) references via the global cosine similarity $\text{CosSim} = \langle \nabla_{\text{relax}}, \nabla_{\text{BP}} \rangle / (\|\nabla_{\text{relax}}\|\|\nabla_{\text{BP}}\|)$, the relative error $\|\nabla_{\text{relax}} - \nabla_{\text{BP}}\| / \|\nabla_{\text{BP}}\|$, and per-layer cosine similarity (additional metrics in Appendix~\ref{sec:appendix_experiments}). Code with fixed seeds is provided in the supplementary material for reproducibility. All experiments ran on a single NVIDIA RTX 6000 Ada (48~GB) with a 32-core CPU and 128~GB DDR5 RAM; each $\eta$ run takes $\sim$2~GPU-hours ($\sim$8~GPU-hours).

\subsection{Results}

Our experiments confirm the theoretical predictions, demonstrating that the proposed method achieves robust stability across a wide range of step sizes, consistent with theoretical bounds.

\textbf{Generalization Performance.}
As shown in Figure~\ref{fig:test_accuracy}, the choice of the discretization step size $\eta$  in Algorithm \ref{alg:xz_schematic} does not negatively impact the learning trajectory. Algorithm~\ref{alg:xz_schematic} achieves test accuracy parity with standard Backpropagation (BP), reaching a peak of roughly $93\%$. This indicates that the discrete-time dynamics successfully approximate the learning signal of the ideal gradient.
\begin{figure}[h]
    \centering
    \includegraphics[width=0.5\linewidth]{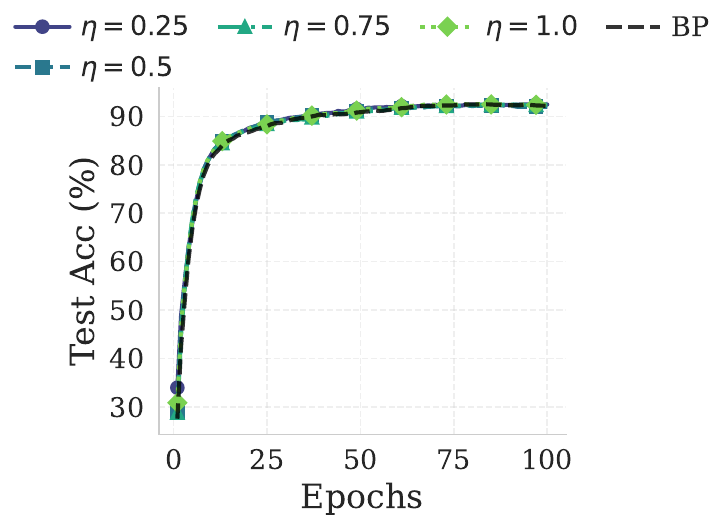}
    \vspace{-3mm}
    \caption{\textbf{Generalization Performance.} Test accuracy evolution for varying step sizes $\eta$ matches the BP baseline (dashed black), reaching $\sim93\%$.}
    \vspace{-3mm}
    \label{fig:test_accuracy}
\end{figure}

\begin{figure}[t]
    \centering
    \includegraphics[width=0.31\linewidth]{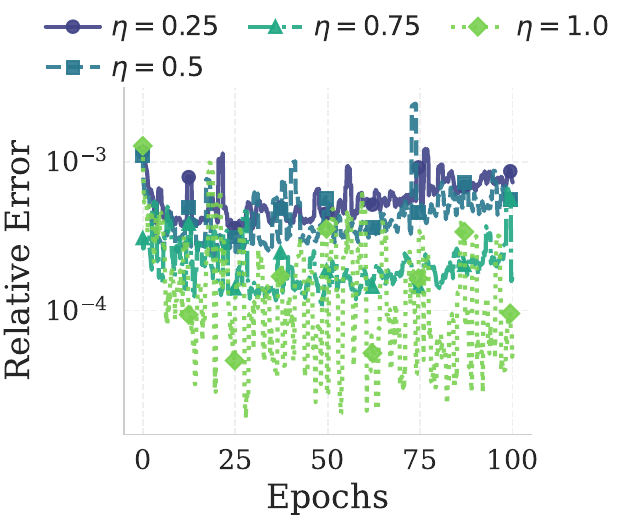}
    \hfill
    \includegraphics[width=0.31\linewidth]{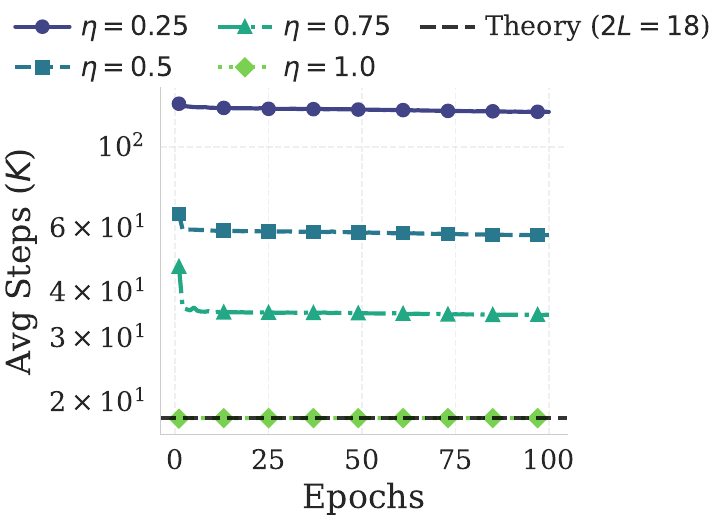}
    \hfill
    \includegraphics[width=0.31\linewidth]{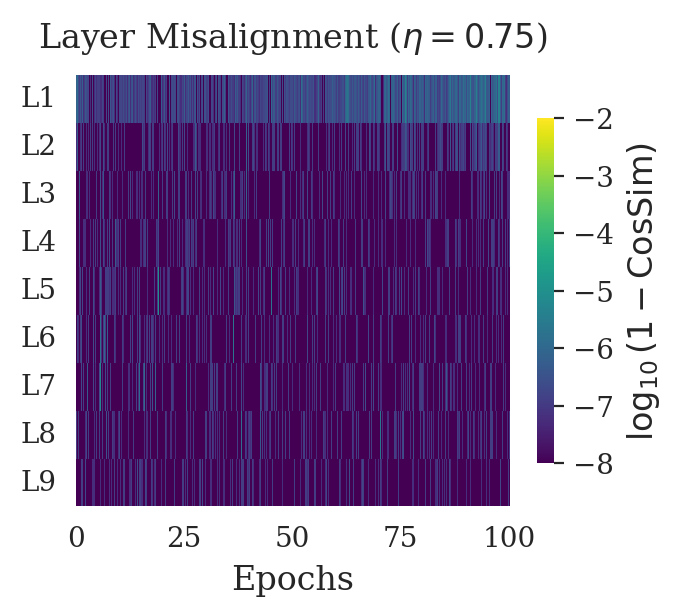}
    \vspace{-3mm}
    \caption{Empirical Validation of Dynamics. (left) \textbf{Relative gradient error} shows accurate recovery of $\nabla_{BP}$. (center) \textbf{Relaxation steps} converge to the theoretical limit $2L=18$ as $\eta\rightarrow1$. (right) \textbf{Log-misalignment heatmap} ($\eta=0.75$) plots $\log_{10}(1 - \text{CosSim})$ for each layer $\text{L}_i$. Values hover between $-6$ and $-8$, bounded by IEEE 754 Float32 machine precision ($\approx 10^{-7}$), rendering relaxation gradients basically indistinguishable from BP.}
    \vspace{-3mm}
    \label{fig:main_results}
\end{figure}

\textbf{Gradient Fidelity and Dynamics.}
Beyond generalization accuracy, we analyze the internal mechanics of the method. Figure~\ref{fig:main_results} corroborates our core theoretical claims along three axes. \emph{Precision} (left panel): the algorithm recovers the true gradient with low global relative error, almost entirely driven by a mismatch in Euclidean norm rather than directional deviation (see Appendix~\ref{app:norm}). \emph{Convergence speed} (center panel): as $\eta \to 1$, the number of required relaxation steps collapses toward the theoretical lower bound $2L=18$ (twice the network depth). \emph{Alignment} (right panel): the layer-wise log-misalignment $\log_{10}(1-\cos\theta)$ hovers near the IEEE 754 Float32 machine precision limit ($\approx 10^{-7}$), making the physical relaxation effectively indistinguishable from the symbolic algorithm (further details in Appendix~\ref{app:layer_cos}).

In sum, these results confirm that Algorithm~\ref{alg:xz_schematic} reliably recovers the exact gradients of standard backpropagation. This validates the central thesis that the physical relaxation process is not an approximation, but the faithful continuous-time generator of the discrete backpropagation update. Further experiments corroborating the accuracy of the method can be found in Appendix~\ref{sec:appendix_experiments}.
% -----------------------------------------------------
% SECTION: RELATED WORK
% -----------------------------------------------------
% -----------------------------------------------------
% -----------------------------------------------------
\section{Related Work}
\label{sec:related}
Our work connects to a broad literature interpreting backpropagation through dynamical and energy-based lenses. Below, we distinguish our formulation from prior approaches.

\begin{itemize}[leftmargin=*, nosep]

\item \textbf{LeCun Formulation of Backpropagation.}
LeCun's formulation \citep{lecun1988theoretical} casts
backpropagation as a constrained optimization problem in which the layer-wise
forward equations are enforced as equality constraints, with Lagrange
multipliers playing the role of the adjoint errors. The resulting solution is
purely algebraic: the multipliers are obtained as the closed-form stationary
point of the constrained Lagrangian, with no associated dynamics. LeCun is
in fact explicit on this point, noting in a footnote that his Lagrangian is
not the Lagrangian of physics, and indeed endowing his formulation with a
symplectic flow would not yield convergence to the target equilibrium, since
a non-dissipative flow, derived from the least action principle, cannot relax into a stationary point. Our framework
resolves this tension at its root: rather than positing a Lagrangian as an
algebraic optimization bookkeeping device, we construct one whose entire dynamical
content follows from the physical principle of least action. Backpropagation
then emerges not as a closed-form solution but as the discrete-time
projection of a continuous physical relaxation toward the same equilibrium,
with the gradient encoded locally in the stress variable $\bm{x}-\bm{z}$ ---
a locality absent from both LeCun's formulation and standard
backpropagation. More broadly, grounding the formulation in the principle of
least action makes the toolkit of theoretical physics natively
applicable to learning.

    \item \textbf{Energy-Based Methods: CHL and Equilibrium Propagation.}
Contrastive Hebbian Learning \citep{movellan1991contrastive} contrasts two
equilibria of an energy-minimizing network --- a free phase and an
output-clamped phase --- and recovers BP-aligned gradients only under
symmetric weights and. Equilibrium Propagation \citep{SB17} instead applies a weak
output nudge of strength $\beta$, recovering exact gradients only in the
limit $\beta\to 0$, again under weight symmetry. Both are also temporally
non-local: the update contrasts quantities at two equilibria reached in
separate phases, requiring storage of phase-one activations until phase
two has settled. The scalar-potential formulation underlying both is
moreover incompatible with feedforward Jacobians, which are non-reciprocal.
\item \textbf{Dual Propagation.}
Dual propagation \citep{DualPropagation-1} represents neuronal activities as
dyads with two oppositely nudged compartments and propagates error signals
via their difference, achieving local credit assignment in a single
inference phase. As shown in \cite{hoier2023lagrangian}, DP is most
faithfully derived as a reformulation of LeCun's Lagrangian framework, and
therefore inherits its structural character: it is an algebraic
constrained-optimization scheme rather than a physical dynamical system,
with no underlying action principle and no relaxation dynamics. Our
framework, by contrast, derives the doubled-variable structure from the
variational principle of least action applied to the network flow itself,
yielding exact gradients in finite time without requiring weak-feedback
limits, and extending naturally to second-order and other physically
motivated dynamics. At the linearised level and for $\alpha=1/2$, the
equilibrium configuration of our $(\bm{x},\bm{z})$ system aligns with that
of dual propagation, clarifying that DP can be viewed as a particular
algebraic shadow of the broader physical relaxation we describe.

\item \textbf{Continuous Adjoint Methods.}
Continuous-time adjoint methods \citep{neuralODEs, chen2018neural} compute
gradients by integrating an adjoint ODE backward in time. Such reverse-time
integration is not realisable physically---feedforward networks included---
since no dissipative substrate evolves against its own arrow of time, and
a contractive forward flow becomes explosively unstable under reversal.
The same obstruction manifests numerically, as instability
\citep{gholami2019anode} and inaccurate gradients \citep{zhuang2020adaptive},
with exactness recovered only as the step size vanishes. Our method instead
applies natively to dissipative, non-reciprocal flows, evolving
$(\bm{x},\bm{z})$ \emph{forward} in time and recovering the exact gradient
locally under unit-step Euler integration.

    \item \textbf{Recurrent Backpropagation.}
Recurrent Backpropagation \citep{almeida1990learning, pineda1987generalization}
computes gradients by solving a separate linear fixed-point equation for
the adjoint variables, engineered as an auxiliary system whose
sole purpose is gradient computation. Our framework instead derives both
the forward dynamics and the adjoint dynamics as the two coupled
Euler--Lagrange equations of a single global Lagrangian, so that inference
and credit assignment are not two systems with one bolted onto the other,
but two facets of one physical relaxation flow.

    \end{itemize}

\paragraph{Limitation: weight transport.}
Our dynamics \ref{eq:xz_saddle_dynamics} retain $\bm{W}^\top$ and thus inherit the weight-transport requirement of standard BP. This is orthogonal to our contribution, which concerns the variational structure of exact gradient computation; combining our framework with feedback alignment \citep{lillicrap2016random, nokland2016direct} or the Kolen--Pollack scheme \citep{kolen1994backpropagation}, is a natural direction we leave to future work.

\noindent\textbf{Summary.}
Unlike methods that rely on asymptotic limits ($\beta \to 0$), weight
symmetry, reverse-time integration, or engineered backward
circuits, our framework is derived end-to-end from the central
physical principle of least action, yielding a \emph{single variational
dynamical system} in which inference and credit assignment unfold together.
This system does not merely approximate the gradient; it frames standard
backpropagation as the exact \emph{discrete-time projection} of a
continuous physical relaxation.

\section{Conclusion}
We have framed backpropagation as the equilibrium of a least-action flow on a
doubled phase space: a single global Lagrangian generates dynamics in which
inference and credit assignment unfold simultaneously and through purely
local interactions, with the exact BP gradient emerging as the stress
between two conjugate fields. Unit-step Euler discretization recovers
standard BP in $2L$ steps, without vanishing-nudge limits, weight symmetry,
reverse-time integration, or engineered backward circuitry; the construction
extends naturally to the full second-order regime (Appendix~\ref{app:mass}).
Beyond unifying inference and gradient computation, this reformulation makes
the toolkit of classical mechanics---Noether's theorem, symplectic geometry,
path-integral extensions---natively applicable to learning, and points
toward analog and neuromorphic substrates in which credit assignment is
carried by the natural relaxation of the hardware itself. The doubled-variable
structure also echoes multi-compartment cortical neurons
\citep{guerguiev2017towards}, hinting at relevance for biological credit
assignment.

% NeurIPS uses the `ack' environment for acknowledgments. It is automatically
% hidden in the anonymized submission and shown in the final / preprint version.
\begin{ack}
AES acknowledges financial support from the Horizon Europe Marie Sk\l odowska-Curie Doctoral Network ``Postdigital Plus'' (Grant 101169118). AES also thanks Serge Massar, Dimitri Vanden Abeele, Bortolo Matteo Mognetti, and Alessandro Salvatore for fruitful discussions and valuable suggestions. Computational resources have been provided by the Consortium des \'Equipements de Calcul Intensif (C\'ECI), funded by the Fonds de la Recherche Scientifique de Belgique (F.R.S.-FNRS) under Grant No.~2.5020.11 and by the Walloon Region.
\end{ack}

% % Visual Separation
% \vspace{2em}
% \hfill \textit{``Vanitas vanitatum et omnia vanitas''}

\newpage

\bibliographystyle{plainnat}
\bibliography{bibliography}

\appendix
\newpage
\section{Overview on Out-of-Equilibrium Mechanics}
\label{app:mechanics_overview}
This appendix provides a brief overview of the classical mechanics of non-conservative systems to motivate the definition of the global Lagrangian functional $\mathcal{L}(\bm{x}, \bm{z})$ presented in Eq.~\ref{eq:energy_xz_icml}. Our construction draws on the formalism developed by \cite{galley2013classical}, \cite{keldysh1964diagram} and \cite{bateman1931dissipative}, adapted here for the specific context of neural dynamics.

\subsection{Lagrangian Mechanics for Dissipative Systems}
\label{app:mechanics_prelims}

We follow the logic of Galley's framework, and refer the reader to the recent
review \cite{aykroyd2025hamiltonian} for a detailed and pedagogical
derivation. Our framework extends the Bateman--Galley formalism, in which an
auxiliary system absorbs the energy lost by a dissipative one so as to
guarantee total conservation, and adapts it to the relaxation dynamics
relevant to feedforward networks.

\paragraph{Conservative, dissipative, and overdamped regimes.}
In a conservative system, the dynamics follow
\begin{equation}
m\ddot{x} = -\frac{\partial U(x)}{\partial x},
\end{equation}
where $m$ is the mass, $U$ is a scalar potential, and the overdot denotes the
time derivative. A dissipative system carries an additional damping term,
\begin{equation}
m\ddot{x} + \gamma\dot{x} + \frac{\partial U(x)}{\partial x} = 0,
\end{equation}
whose dynamics is recovered from the modified Lagrange equation
\begin{equation}
\frac{d}{dt}\!\left(\frac{\partial L}{\partial \dot{x}}\right)
-\frac{\partial L}{\partial x}
+\frac{\partial R}{\partial \dot{x}}
=0,
\label{eq:lagr_with_rayleigh}
\end{equation}
where $R=\tfrac{1}{2}\dot{x}^{2}$ is the Rayleigh dissipation function for
linear damping (we set $\gamma=1$). In machine learning, and in
equilibrium-based theories such as ours, the natural setting is one of
first-order dynamics. Physics calls this the \emph{overdamped} regime (e.g.\
Brownian motion): the acceleration term is neglected and one recovers
\begin{equation}
\dot{x} = -\frac{1}{\gamma}\frac{\partial U(x)}{\partial x}.
\end{equation}
We adopt this regime throughout the main text purely for clarity of
exposition; the construction extends without difficulty to the full massive
case, recovering the same equilibria and stability properties, as briefly
shown in Appendix~\ref{app:mass}. In this regime,
\eqref{eq:lagr_with_rayleigh} simplifies to
\begin{equation}
-\frac{\partial L}{\partial x} + \frac{\partial R}{\partial \dot{x}} = 0
\quad\Longleftrightarrow\quad
\dot{x} = \frac{\partial L}{\partial x}.
\end{equation}

\paragraph{Doubling and the action principle.}
To accommodate non-reciprocal forces, the Bateman--Galley formalism doubles
the configuration space, replacing $x$ by a conjugate pair $(x,z)$. Because
we work in the overdamped regime, the doubled Lagrangian carries no kinetic
term, and the natural action takes the form
\begin{equation}
S=\int_{t_1}^{t_2}\mathcal{L}(x,\dot{x},z,\dot{z})\,dt
=\int_{t_1}^{t_2}\!\big[L(x,\dot{x})-L(z,\dot{z})\big]\,dt,
\end{equation}
in which the $z$-path is integrated backward in time relative to the
$x$-path; this is the origin of the sign asymmetry that will appear in the
equations of motion.

\paragraph{Antisymmetric coupling.}
Following \cite{aykroyd2025hamiltonian} and the Galley construction, we
introduce an antisymmetric coupling term $K(x,z)$ between the two paths,
\begin{equation}
S=\int_{t_1}^{t_2}\!\big[L(x,\dot{x})-L(z,\dot{z})+K(x,\dot{x},z,\dot{z})\big]\,dt,
\label{eq:doubled_action_K}
\end{equation}
which breaks the path symmetry and injects non-conservative effects into the
equations of motion.  Applying \eqref{eq:lagr_with_rayleigh} to \eqref{eq:doubled_action_K}
yields
\begin{equation}
\frac{d}{dt}\!\left(\frac{\partial \mathcal{L}}{\partial \dot{x}}\right)+\dot{x}=\frac{\partial \mathcal{L}}{\partial x},
\qquad
-\frac{d}{dt}\!\left(\frac{\partial \mathcal{L}}{\partial \dot{z}}\right) +\dot{z}=-\frac{\partial \mathcal{L}}{\partial z}.
\end{equation}The construction of the explicit form of $K$ and the generalized Lagrangian $\mathcal{L}$ for the
network flow $F$ is carried out in
Appendix~\ref{app:lagr_build}.

\subsection{From a Non-Reciprocal Flow to a Variational Principle}
\label{app:nonreciprocal_to_variational}

The forward dynamics
$\dot{\bm{a}} = F(\bm{a}) = \bm{\sigma}(\bm{W}\bm{a} + \bm{\beta}(\bm{x}_0)) - \bm{a}$
introduced in Sec.~\ref{subsec: forward_vf} is intrinsically
\emph{non-reciprocal}: information flows from layer $\ell$ to $\ell+1$ but
not in reverse, and the corresponding Jacobian $\nabla F$ is strictly
block-lower-triangular and therefore non-self-adjoint. The Helmholtz
conditions of the inverse problem of the calculus of variations state that
a vector field admits a scalar potential---i.e., can be written as
$F = -\nabla H$ for some $H$---if and only if its Jacobian is symmetric.
Consequently, no scalar potential exists on the original $n$-dimensional
state space whose gradient generates the feedforward flow, and the standard
``energy minimisation'' route to a variational principle is unavailable.

The Bateman--Galley formalism
circumvents this obstruction by doubling the configuration space. The
rotational (non-reciprocal) component of $F$, which has no scalar potential
in the original $n$-dimensional space, is absorbed into a bilinear coupling
between two conjugate fields $\bm{x},\bm{z}\in\mathbb{R}^{n}$ on the
doubled $2n$-dimensional space, where it admits a well-defined Lagrangian
description. The physical motion of the original system is recovered on
the diagonal submanifold $\bm{x}=\bm{z}$, while the off-diagonal direction
$\bm{x}-\bm{z}$ provides the additional degree of freedom needed to encode
the non-reciprocity variationally. We adapt this construction to the
network flow $F$ in Appendix~\ref{app:lagr_build} below.

\subsection{Constructing the Lagrangian}
\label{app:lagr_build}

We now make explicit how the global Lagrangian \eqref{eq:energy_xz_icml} is
obtained from the bare forward vector field $F$. The construction follows the
inverse-problem recipe presented in Section~V of \cite{aykroyd2025hamiltonian}
for arbitrary second-order ODEs, specialised here to the first-order setting
that is natural for feedforward relaxation dynamics.

\paragraph{Notation}
In this section we are briefly describing how to construct the Lagrangian in our specific case; we will map our framework (even if we initially obtained the formulation through a geometrical argument) onto the construction in section V of \cite{aykroyd2025hamiltonian} to give the reader more material and better clarity. In our formalism the conjugate pair
$(\bm{x},\bm{z})\in\mathbb{R}^{2n}$, which corresponds to the
$(\bm{q}_\uparrow,\bm{q}_\downarrow)$ parametrisation of
\cite{aykroyd2025hamiltonian}. The change of basis to
$\big(\tfrac{1}{2}(\bm{x}+\bm{z}),\,\bm{x}-\bm{z}\big)$ then corresponds to
their $(\bm{q}_+,\bm{q}_-)$ parametrisation, with $\tfrac{1}{2}(\bm{x}+\bm{z})$
playing the role of the physical (plus) variable and $\bm{x}-\bm{z}$ the role
of the virtual (minus) variable. The physical limit corresponds to
$\bm{x}=\bm{z}$.

\paragraph{Reconstruction from a target ODE.}
For a generic second-order target ODE
$\ddot{\bm{q}}=\bm{U}(\bm{q},\dot{\bm{q}},t)$,
\cite{aykroyd2025hamiltonian} show that the doubled-variable Lagrangian
\begin{equation}
\mathcal{L}
=
\bm{q}_-\!\cdot\!\big(\ddot{\bm{q}}_+ - \bm{U}(\bm{q}_+,\dot{\bm{q}}_+,t)\big)
\label{eq:lagr_trivial}
\end{equation}
trivially reproduces the equations of motion as its physical-limit
Euler--Lagrange equation. Integrating \eqref{eq:lagr_trivial} by parts removes
the explicit acceleration and yields the equivalent linear (Lie) form
\begin{equation}
\mathcal{L}(\bm{q}_\pm,\dot{\bm{q}}_\pm,t)
=
\dot{\bm{q}}_-\!\cdot\!\dot{\bm{q}}_+
+
\bm{q}_-\!\cdot\!\bm{U}(\bm{q}_+,\dot{\bm{q}}_+,t),
\label{eq:lagr_secondorder}
\end{equation}
which is Eq.~(121) of \cite{aykroyd2025hamiltonian}. The first term is the
canonical kinetic coupling associated with the standard momentum choice
$\bm{P}=\dot{\bm{q}}_+$ (see Eq.~(126) of \cite{aykroyd2025hamiltonian}); the
second encodes the force $\bm{U}$.

\paragraph{Specialisation to the network flow.}
The forward relaxation defined in Sec.~\ref{subsec: forward_vf},
\begin{equation}
\frac{d\bm{a}}{dt}
=
F(\bm{a})
=
\bm{\sigma}(\bm{W}\bm{a}+\bm{\beta}(\bm{x}_0))-\bm{a},
\end{equation}
is first-order in time. Adapting \eqref{eq:lagr_secondorder} to this regime
amounts to neglecting the acceleration (kinetic) term
$\dot{\bm{q}}_-\!\cdot\!\dot{\bm{q}}_+$ and replacing the second-order
acceleration field $\bm{U}$ by the first-order velocity field $F$. This is no
loss of generality: the analogous derivation retaining a mass term reaches the
same equilibria and stability properties, and is treated separately in
Appendix~\ref{app:mass}. Translating to the $(\bm{x},\bm{z})$ notation, the
remaining contribution is
\begin{equation}
\mathcal{L}(\bm{x},\bm{z})
=
(\bm{x}-\bm{z})^{\!\top}
F\!\left(\frac{\bm{x}+\bm{z}}{2}\right),
\label{eq:lagr_xz_clean}
\end{equation}
which is precisely the Lagrangian appearing in
\eqref{eq:energy_xz_icml}, before the symmetry-breaking cost is added. The
auxiliary variable $\bm{z}$ then plays exactly the role of the virtual copy in
the Galley framework: it absorbs the rotational (non-conservative) component
of $F$ off the physical submanifold $\bm{x}=\bm{z}$, while leaving the on-shell
motion unchanged. Adding the task loss
$C\!\big([\tfrac{1}{2}(\bm{x}+\bm{z})]_L,\bm{y}\big)$ to
\eqref{eq:lagr_xz_clean} as a symmetry-breaking perturbation yields the full
global Lagrangian \eqref{eq:energy_xz_icml} used throughout the paper.

\subsection{Symmetry Breaking as Credit Assignment}
\label{app:symmetry_breaking}
The variational principle established above successfully generates the forward dynamics on the diagonal manifold $\bm{x} = \bm{z}$. On this manifold, the system possesses a gauge symmetry: the forward state $\bm{z}$ is redundant, and the stress $(\bm{x}-\bm{z})$ is zero.

We propose to exploit this symmetry to perform credit assignment. Specifically, we introduce the task loss $C(\bm{m})$ as a symmetry-breaking potential acting on the output block. By adding this cost to the Lagrangian, we intentionally perturb the gauge symmetry of the system.

This perturbation exerts opposite forces on the pair: it drives $\bm{z}$ to descend the loss landscape while forcing $\bm{x}$ to ascend.

This "tug-of-war" induces a physical separation $\bm{s} = \bm{x} - \bm{z}$. The evolution of this stress variable naturally recovers the exact gradients. Thus, in our framework, backpropagation emerges not as a symbolic operation, but as the physical relaxation of the stress induced by explicitly breaking the symmetry of the forward pass.

\section{Extension to Second-Order Dynamics}
\label{app:mass}

In the main text we worked in the first-order regime, in which the
acceleration term of the Aykroyd--Galley Lagrangian is dropped. We show here
that this is purely an expository choice: the construction extends without
modification to the full second-order (massive) case, and recovers the same
equilibria and the same gradient signal.

\paragraph{Massive Lagrangian.}
Endowing the doubled state $(\bm{x},\bm{z})$ with a mass $m$ and reinstating
the kinetic coupling of Eq.~(121) in \cite{aykroyd2025hamiltonian} yields the
global Lagrangian
\begin{equation}
\mathcal{L}(\bm{x},\bm{z},\dot{\bm{x}},\dot{\bm{z}})
=
m\,(\dot{\bm{x}}-\dot{\bm{z}})^{\!\top}\!
\left(\frac{\dot{\bm{x}}+\dot{\bm{z}}}{2}\right)
+
(\bm{x}-\bm{z})^{\!\top}
F\!\left(\frac{\bm{x}+\bm{z}}{2}\right)
+
C\!\left(\!\left[\frac{\bm{x}+\bm{z}}{2}\right]_L\!,\bm{y}\right),
\label{eq:lagr_mass}
\end{equation}
which differs from \eqref{eq:energy_xz_icml} only by the leading kinetic term.
The remaining ingredients --- the bilinear lifting of the network flow $F$
and the symmetry-breaking cost --- are unchanged.

\paragraph{Equations of motion.}
Extremising the action of \eqref{eq:lagr_mass} together with the Rayleigh
dissipation function $R=\tfrac{1}{2}(\dot{\bm{x}}^{2}+\dot{\bm{z}}^{2})$ gives
\begin{equation}
\begin{aligned}
\frac{d}{dt}\!\left(\frac{\partial \mathcal{L}}{\partial \dot{\bm{x}}}\right)
-\frac{\partial \mathcal{L}}{\partial \bm{x}}
&=
-\frac{\partial R}{\partial \dot{\bm{x}}}
\;\Longrightarrow\;
m\ddot{\bm{x}}+\dot{\bm{x}}
=
\frac{\partial \mathcal{L}}{\partial \bm{x}},
\\[4pt]
-\!\left(\frac{d}{dt}\!\left(\frac{\partial \mathcal{L}}{\partial \dot{\bm{z}}}\right)
-\frac{\partial \mathcal{L}}{\partial \bm{z}}\right)
&=
-\frac{\partial R}{\partial \dot{\bm{z}}}
\;\Longrightarrow\;
m\ddot{\bm{z}}+\dot{\bm{z}}
=
-\frac{\partial \mathcal{L}}{\partial \bm{z}}.
\end{aligned}
\label{eq:eom_mass}
\end{equation}
In the limit $m\to 0$ these collapse exactly to the first-order
flow used in the main text. For $m>0$, \eqref{eq:eom_mass} is a damped
second-order system whose first-order limit recovers our dynamics verbatim.

\paragraph{Equilibria and gradient signal.}
At any equilibrium $(\bar{\bm{x}},\bar{\bm{z}})$ of \eqref{eq:eom_mass} we
have $\ddot{\bm{x}}=\dot{\bm{x}}=\bm{0}$ and
$\ddot{\bm{z}}=\dot{\bm{z}}=\bm{0}$, so the kinetic and damping terms drop
out and the equilibrium conditions reduce to
\begin{equation}
\frac{\partial \mathcal{L}}{\partial \bm{x}}\bigg|_{\bar{\bm{x}},\bar{\bm{z}}}
=\bm{0},
\qquad
\frac{\partial \mathcal{L}}{\partial \bm{z}}\bigg|_{\bar{\bm{x}},\bar{\bm{z}}}
=\bm{0},
\end{equation}
which are exactly the equilibrium conditions analysed in
Sec.~\ref{sec:m-s}. Consequently the equilibrium configuration
$(\bar{\bm{m}},\bar{\bm{s}})$ is unchanged, and the gradient extracted from
the stress variable $\bar{\bm{s}}=\bar{\bm{x}}-\bar{\bm{z}}$ remains the
exact backpropagation gradient \eqref{eq:backprop_s}.

\paragraph{Stability.}
The Jacobian of the linearised second-order system inherits the spectral
structure of the first-order analysis of Appendix~\ref{subsec:stability},
augmented by the standard damped-oscillator eigenvalue pattern: each
eigenvalue $\lambda=-1$ of the first-order Jacobian becomes a pair of roots
of the characteristic polynomial $m\mu^{2}+\mu+1=0$, whose real parts are
strictly negative for any $m>0$. The equilibrium therefore remains globally
exponentially stable, with a relaxation rate that interpolates smoothly
between the overdamped and underdamped regimes as $m$ is varied.

In summary, the massive extension preserves both the forward-pass equilibria
and the backpropagation gradients, and merely modifies the transient
relaxation profile. The first-order treatment of the main text is therefore
a notational simplification rather than a structural restriction.

\section{Algorithm}
\label{app:dbp_alg}
The pseudocode for the integration of the Dyadic Backpropagation scheme is given below.

\begin{algorithm}[H]
\caption{Dyadic Backpropagation (DBP)}
\label{alg:xz_schematic}
\small
\begin{algorithmic}[1]

\STATE \textbf{Input:} $\bm{x}_0,\bm{y}$, parameters $\{\bm{W}_\ell,\bm{b}_\ell\}_{\ell=1}^L$, step size $\eta>0$, tolerance $\varepsilon>0$
\STATE Initialize $\bm{x}^0 = \bm{z}^0$

\FOR{$k=0,1,2,\dots$}
    \STATE $\bm{m}^k \leftarrow \tfrac{1}{2}(\bm{x}^k + \bm{z}^k)$
    \STATE $\bm{D}^k \leftarrow \operatorname{diag}\!\big(\bm{\sigma}'(\bm{W}\bm{m}^k + \bm{\beta}(\bm{x}_0))\big)$
    \STATE Compute loss gradient $\bm{g}^k \leftarrow \nabla_{\bm{m}} C(\bm{m}^k,\bm{y})$, embedded on the output block

    \STATE Compute velocity fields from Eqs.~\eqref{eq:xz_saddle_dynamics}
    \[
    \begin{aligned}
    \dot{\bm{x}}^k &= \bm{\sigma}(\bm{W}\bm{m}^k + \bm{\beta}(\bm{x}_0)) - \bm{x}^k \\
                   &\quad + \tfrac{1}{2}\bm{W}^\top \bm{D}^k (\bm{x}^k - \bm{z}^k) + \tfrac{1}{2}\bm{g}^k, \\
    \dot{\bm{z}}^k &= \bm{\sigma}(\bm{W}\bm{m}^k + \bm{\beta}(\bm{x}_0)) - \bm{z}^k \\
                   &\quad - \tfrac{1}{2}\bm{W}^\top \bm{D}^k (\bm{x}^k - \bm{z}^k) - \tfrac{1}{2}\bm{g}^k
    \end{aligned}
    \]

    \STATE Update states by Euler integration:
    \[
    \bm{x}^{k+1} \leftarrow \bm{x}^k + \eta\,\dot{\bm{x}}^k,
    \qquad
    \bm{z}^{k+1} \leftarrow \bm{z}^k + \eta\,\dot{\bm{z}}^k
    \]

    \IF{$\|\bm{x}^{k+1}-\bm{x}^k\|_2 + \|\bm{z}^{k+1}-\bm{z}^k\|_2 < \varepsilon$}
        \STATE \textbf{break}
    \ENDIF
\ENDFOR

\STATE \textbf{Equilibrium variables:}
\[
\bar{\bm{x}} = \bm{x}^k,\quad
\bar{\bm{z}} = \bm{z}^k,\quad
\bar{\bm{m}} = \tfrac{1}{2}(\bar{\bm{x}}+\bar{\bm{z}}),\quad
\bar{\bm{s}} = \bar{\bm{x}}-\bar{\bm{z}}.
\]
\STATE \textbf{Gradient (equivalent to standard BP):}
\[
\frac{dC}{d\bm{W}}
=
\bm{D}(\bar{\bm{m}})\,\bar{\bm{s}}\,\bar{\bm{m}}^{\!\top}
=
\bm{D}\!\left(\tfrac{\bar{\bm{x}}+\bar{\bm{z}}}{2}\right)
(\bar{\bm{x}}-\bar{\bm{z}})
\left(\tfrac{\bar{\bm{x}}+\bar{\bm{z}}}{2}\right)^{\!\top}.
\]

\end{algorithmic}
\end{algorithm}

% \section{Equilibrium Lagrangian equals task loss.}
% \label{app:eq_en}
% It is instructive to compute the value of the Lagrangian $\mathcal{L}(\bm{m}, \bm{s})$ at the equilibrium point $(\bar{\bm{m}}, \bar{\bm{s}})$. In the forward phase, the primal variable $\bm{m}$ relaxes to the exact fixed point of the network equations:
% \begin{equation}
% \bar{\bm{m}} = \bm{\sigma}(\bm{W}\bar{\bm{m}} + \bm{\beta}(\bm{x}_0)).
% \end{equation}
% Substituting this into $\mathcal{L}(\bm{m}, \bm{s})$ yields:
% \begin{align}
% \mathcal{L}(\bar{\bm{m}}, \bar{\bm{s}})
% &=
% \bar{\bm{s}}^{\!\top}
% \underbrace{\Big[
% \bm{\sigma}\!\big(\bm{W}\bar{\bm{m}} + \bm{\beta}(\bm{x}_0)\big)
% - \bar{\bm{m}}
% \Big]}_{\bm{0}}
% + C(\bar{\bm{m}}_L,\bm{y}) \\
% &= C(\bar{\bm{m}}_L,\bm{y}).
% \end{align}
% This result confirms that the saddle-point equilibrium is physically meaningful: the ``stored work'' in the constraints ($\bm{s}^\top [\dots]$) vanishes when the forward pass is valid, leaving only the potential energy of the output error. The system thus relaxes to a state where the energy is exactly the task loss, constrained to the manifold of valid network activations.

\section{Proof of Forward Layer Freezing}
\label{app:forward_freezing_proof}
We provide the deferred proof of the Forward Layer Freezing lemma stated in Section~\ref{sec:discrete}.

\begin{proof}
The proof proceeds by induction on $\ell$.
For $\ell=1$, the update reads
\[
\bm{m}_1^{k+1} = \sigma_1(\bm{W}_1 \bm{x}_0 + \bm{b}_1),
\]
which is independent of $k$; hence $\bm{m}_1^{k}=\bm{a}_1$ for all $k\ge 1$.
Now assume the claim holds for layers $1,\dots,\ell-1$.
The $\ell$-th block update uses only $\bm{m}_{\ell-1}^k$, which equals $\bm{a}_{\ell-1}$ for all $k\ge \ell-1$.
Thus for $k\ge \ell-1$,
\[
\bm{m}_\ell^{k+1} = \sigma_\ell(\bm{W}_\ell \bm{m}_{\ell-1}^k + \bm{b}_\ell)
= \sigma_\ell(\bm{W}_\ell \bm{a}_{\ell-1} + \bm{b}_\ell)
= \bm{a}_\ell,
\]
and the claim follows.
\end{proof}

\section{Algorithmic Form of the Discrete Dynamics}
\label{app:autonomous_alg}
We give here the explicit pseudocode for the discrete-time projection of the $(\bm{m},\bm{s})$ dynamics referenced in Section~\ref{sec:discrete}.

\begin{algorithm}[H]
\caption{Dynamical Backpropagation}
\label{alg:autonomous_bp}
\small
\begin{algorithmic}[1]
\STATE \textbf{Input:} $\bm{x}_0, \bm{y}, \{\bm{W}_\ell,\bm{b}_\ell\}_{\ell=1}^L$
\STATE \textbf{Initialize:} $\bm{m}^0$, $\bm{s}^0$
\STATE \emph{// Simultaneous evolution loop}
\FOR{$k=0$ to $2L-1$}
    \STATE $\bm{m}^{k+1} \leftarrow \bm{\sigma}(\bm{W}\bm{m}^k + \bm{\beta}(\bm{x}_0))$
    \STATE $\bm{D}^k \leftarrow \operatorname{diag}\!\big(\bm{\sigma}'(\bm{W}\bm{m}^k + \bm{\beta}(\bm{x}_0))\big)$
    \STATE $\bm{s}^{k+1} \leftarrow \bm{W}^\top \bm{D}^k \bm{s}^k + \nabla_{\bm{m}_L}C(\bm{m}_L^k, \bm{y})$
\ENDFOR
\STATE \textbf{Gradient:} $\displaystyle \frac{\partial C}{\partial \bm{W}} = \bm{D}(\bm{m}^{2L})\,\bm{s}^{2L}\,(\bm{m}^{2L})^{\top}$
\end{algorithmic}
\end{algorithm}

\section{Properties of the Equilibrium}
In this section we analyze some claims of the main text related to the properties of the equilibrium configuration of our relaxation dynamics.
\subsection{Explicit Equivalence to Backpropagation}
\label{subsec:bp_eq}

We derive the explicit form of the equilibrium state $\bar{\bm{s}}$ by directly inverting the linear system governing the backward phase. Recall the equilibrium condition for the stress variable from Eq.~\eqref{eq:s_equilibrium_linear_full}:
\begin{equation}
(\bm{I} - \bm{W}^\top \bar{\bm{D}})\,\bar{\bm{s}} = \nabla_{\bm{m}} C,
\end{equation}
where $\nabla_{\bm{m}} C$ denotes the gradient of the loss with respect to all activations. Due to the network structure, this vector is sparse, containing non-zero entries only in the block corresponding to the output layer $L$.

Let $\bm{M} = \bm{W}^\top \bar{\bm{D}}$. Since $\bm{W}$ is strictly lower block-triangular (representing feedforward connectivity), its transpose $\bm{W}^\top$ is strictly upper block-triangular. The multiplication by the diagonal matrix $\bar{\bm{D}}$ preserves this structure.

Consequently, $\bm{M}$ acts as a ``backward shift'' operator, moving signals from layer $\ell$ to $\ell-1$. Crucially, this implies that $\bm{M}$ is nilpotent with index $L$:
\begin{equation}
\bm{M}^L = \bm{0}.
\end{equation}
Physically, this reflects the fact that error signals cannot propagate backward through more than $L$ layers.

Because $\bm{M}$ is nilpotent, the matrix $(\bm{I} - \bm{M})$ is invertible, and its inverse is given exactly by the finite Neumann series:
\begin{equation}
(\bm{I} - \bm{M})^{-1} = \sum_{k=0}^{L-1} \bm{M}^k = \bm{I} + \bm{M} + \bm{M}^2 + \dots + \bm{M}^{L-1}.
\end{equation}
Substituting this into the equilibrium equation yields the closed-form solution:
\begin{equation}
\bar{\bm{s}} = \left( \sum_{k=0}^{L-1} (\bm{W}^\top \bar{\bm{D}})^k \right) \nabla_{\bm{m}} C.
\end{equation}

\paragraph{Interpretation.}
This series expansion recovers the standard backpropagation of errors. The term $k=0$ captures the direct loss gradient at the output. The term $k=1$ represents the loss gradient backpropagated by one layer ($\bm{W}^\top \bar{\bm{D}} \nabla_{\bm{m}} C$), and the general term $k$ represents the contribution of the loss backpropagated through $k$ layers. Since $\nabla_{\bm{m}} C$ is non-zero only at layer $L$, the vector $\bar{\bm{s}}$ correctly aggregates the total derivative of the loss with respect to the activations at every layer $\ell$.

\subsection{Exponential Stability of the Equilibrium}
\label{subsec:stability}

We analyze the local stability of the coupled dynamical system by inspecting the spectrum of its Jacobian matrix at the equilibrium point $(\bar{\bm{m}}, \bar{\bm{s}})$.

Let $\bm{u} = [\bm{m}^\top, \bm{s}^\top]^\top$ denote the full state vector. The total Jacobian $\mathcal{J} \in \mathbb{R}^{2n \times 2n}$ is given by the block matrix:
\begin{equation}
\mathcal{J} =
\begin{bmatrix}
\frac{\partial \dot{\bm{m}}}{\partial \bm{m}} & \frac{\partial \dot{\bm{m}}}{\partial \bm{s}} \\[6pt]
\frac{\partial \dot{\bm{s}}}{\partial \bm{m}} & \frac{\partial \dot{\bm{s}}}{\partial \bm{s}}
\end{bmatrix}.
\end{equation}

From the system definition in Eq.~\eqref{eq:ms_dynamics}, we observe that the evolution of the primal variable $\bm{m}$ is independent of the dual variable $\bm{s}$. Consequently, the upper-right block vanishes:
\[
\frac{\partial \dot{\bm{m}}}{\partial \bm{s}} = \bm{0}.
\]
This block-triangular structure implies that the spectrum of $\mathcal{J}$ is the union of the spectra of its diagonal blocks. We analyze them individually:

\textbf{1. Primal Block ($\bm{m}$).} The linearization of the forward dynamics is
\[
\bm{J}_{\bm{m}\bm{m}} = \frac{\partial \dot{\bm{m}}}{\partial \bm{m}} = \operatorname{diag}(\bm{\sigma}'(\bm{m}))\bm{W} - \bm{I}.
\]
Since $\bm{W}$ is strictly lower block-triangular (nilpotent), the product $\operatorname{diag}(\dots)\bm{W}$ retains this strictly lower triangular structure and has zeros on its main diagonal. Therefore, the eigenvalues are determined solely by the identity term:
\[
\operatorname{spec}(\bm{J}_{\bm{m}\bm{m}}) = \{-1\}.
\]

\textbf{2. Dual Block ($\bm{s}$).} The linearization of the backward dynamics with respect to $\bm{s}$ is
\[
\bm{J}_{\bm{s}\bm{s}} = \frac{\partial \dot{\bm{s}}}{\partial \bm{s}} = \bm{W}^\top \operatorname{diag}(\bm{\sigma}'(\bm{m})) - \bm{I}.
\]
Since $\bm{W}$ is strictly lower triangular, its transpose $\bm{W}^\top$ is strictly upper triangular. The product $\bm{W}^\top \operatorname{diag}(\dots)$ is thus strictly upper triangular with a zero diagonal. Similarly, the eigenvalues are determined by the identity term:
\[
\operatorname{spec}(\bm{J}_{\bm{s}\bm{s}}) = \{-1\}.
\]

\textbf{Conclusion.} The Jacobian $\mathcal{J}$ has a single eigenvalue $\lambda = -1$ with algebraic multiplicity $2n$. Since $\operatorname{Re}(\lambda) < 0$, the equilibrium $(\bar{\bm{m}}, \bar{\bm{s}})$ is \textbf{globally exponentially stable}. The system relaxes with a natural time constant $\tau = 1$, ensuring rapid convergence to the backpropagation solution regardless of initialization. Finally, the $(\bm{x},\bm{z})$ system also possesses globally exponentially stable equilibria, given that these dynamics are obtained through a linear change of variable.
\section{Alternative Formulation: Split-Dynamics}
\label{app:split_dynamics}

To facilitate implementation on hardware, we present an alternative dynamical formulation where the forward drive and Jacobian operators are evaluated individually at $\bm{x}$ and $\bm{z}$ rather than at the shared midpoint, neglecting all non-linear orders of the stress term $(x-z)$ in Eq.(\ref{eq:xz_saddle_dynamics}). One can easily see that only the linear term of the stress term $(x-z)$ will contribute to the gradient.

\subsection{Split-Dynamics Equations}

Using the same notation of the main text, neglecting all non-linear orders of $(x-z)$ in Eq.(\ref{eq:xz_saddle_dynamics}), we get

\begin{align}
\frac{d\bm{x}}{dt}
&=
\frac{1}{2}\Big[ F(\bm{x}) + F(\bm{z}) \Big]
+ \frac{1}{2} \left(\frac{\partial F}{\partial \bm{x}}\right)^T (\bm{x}-\bm{z})
+ \frac{1}{2} \nabla_{\bm{x}} C,
\label{eq:split_dx}
\\
\frac{d\bm{z}}{dt}
&=
\frac{1}{2}\Big[ F(\bm{x}) + F(\bm{z}) \Big]
- \frac{1}{2} \left(\frac{\partial F}{\partial \bm{z}}\right)^T (\bm{x}-\bm{z})
- \frac{1}{2} \nabla_{\bm{z}} C.
\label{eq:split_dz}
\end{align}

Here, $\nabla_{\bm{u}} C$ denotes the gradient of the loss injected at the output block of the state $\bm{u}$.

This formulation avoids the computation of the midpoint state $\bm{m}$ for derivative evaluations.

\subsection{Split-Jacobian Relaxation Algorithm}

The corresponding discrete-time implementation of the Split-Dynamics is provided in Algorithm \ref{alg:split_dynamics_schematic}. This reformulation of the algorithm, tested on the same tasks, matches the performance of algorithm \ref{alg:xz_schematic} as predicted from the theory.

\begin{algorithm}[t]
\caption{Decoupled Split-Jacobian $\bm{x}$--$\bm{z}$ Relaxation}
\label{alg:split_dynamics_schematic}
\small
\begin{algorithmic}[1]

\STATE \textbf{Input:} $\bm{x}_0,\bm{y}$, parameters $\{\bm{W}_\ell,\bm{b}_\ell\}_{\ell=1}^L$, step size $\eta>0$, tolerance $\varepsilon>0$
\STATE Initialize $\bm{x}^0 = \bm{0}$, $\bm{z}^0 = \bm{0}$

\FOR{$k=0,1,2,\dots$}
    \STATE \textbf{Define local variables:}
    \[
    \begin{aligned}
    \bm{s}^k &= \bm{x}^k - \bm{z}^k, \\
    \bm{\Sigma}^k &= \tfrac{1}{2}\big[\bm{\sigma}(\bm{W}\bm{x}^k + \bm{\beta}) + \bm{\sigma}(\bm{W}\bm{z}^k + \bm{\beta})\big], \\
    \bm{D}_{\bm{x}}^k &= \operatorname{diag}\!\big(\bm{\sigma}'(\bm{W}\bm{x}^k + \bm{\beta})\big), \quad \bm{g}_{\bm{x}}^k = \nabla_{\bm{x}} C(\bm{x}^k, \bm{y}), \\
    \bm{D}_{\bm{z}}^k &= \operatorname{diag}\!\big(\bm{\sigma}'(\bm{W}\bm{z}^k + \bm{\beta})\big), \quad \bm{g}_{\bm{z}}^k = \nabla_{\bm{z}} C(\bm{z}^k, \bm{y}).
    \end{aligned}
    \]

    \STATE Compute velocity fields (fully decoupled dynamics):
    \[
    \begin{aligned}
    \dot{\bm{x}}^k &= \bm{\Sigma}^k - \bm{x}^k
    + \tfrac{1}{2}\bm{W}^\top \bm{D}_{\bm{x}}^k \bm{s}^k
    + \tfrac{1}{2}\bm{g}_{\bm{x}}^k,
    \\
    \dot{\bm{z}}^k &= \bm{\Sigma}^k - \bm{z}^k
    - \tfrac{1}{2}\bm{W}^\top \bm{D}_{\bm{z}}^k \bm{s}^k
    - \tfrac{1}{2}\bm{g}_{\bm{z}}^k
    \end{aligned}
    \]

    \STATE Update states:
    \[
    \bm{x}^{k+1} \leftarrow \bm{x}^k + \eta\,\dot{\bm{x}}^k,
    \qquad
    \bm{z}^{k+1} \leftarrow \bm{z}^k + \eta\,\dot{\bm{z}}^k
    \]

    \IF{$\|\bm{x}^{k+1}-\bm{x}^k\|_2 + \|\bm{z}^{k+1}-\bm{z}^k\|_2 < \varepsilon$}
        \STATE \textbf{break}
    \ENDIF
\ENDFOR

\STATE \textbf{Equilibrium variables:}
\[
\bar{\bm{m}} = \tfrac{1}{2}(\bm{x}^k+\bm{z}^k),\quad
\bar{\bm{s}} = \bm{x}^k-\bm{z}^k.
\]

\STATE \textbf{Gradient computation:}
\[
\frac{dC}{d\bm{W}}
=
\bm{D}(\bar{\bm{m}})\,\bar{\bm{s}}\,\bar{\bm{m}}^{\!\top}.
\]

\end{algorithmic}
\end{algorithm}
\section{Additional Experimental Results}
\label{sec:appendix_experiments}

In this section, we provide additional experimental results corroborating the gradient fidelity of the proposed Algorithm \ref{alg:xz_schematic}. We compare the relaxation-derived gradients against reference gradients computed via standard backpropagation (BP) on the CIFAR-10 dataset.

\subsection{Training Stability and Loss Trajectory}
To validate that the relaxation dynamics do not introduce optimization artifacts, we tracked the training loss across different Euler step sizes $\eta$. As illustrated in Figure \ref{fig:trainingloss}, the learning trajectories for $\eta \in \{0.25, 0.5, 0.75, 1.0\}$ indistinguishably overlap with the ideal case $\eta=1$.
This further confirms that the our formulation effectively solves the credit assignment problem without degrading optimization stability.

\begin{figure}[h]
  \centering
  \includegraphics[width=0.5\textwidth]{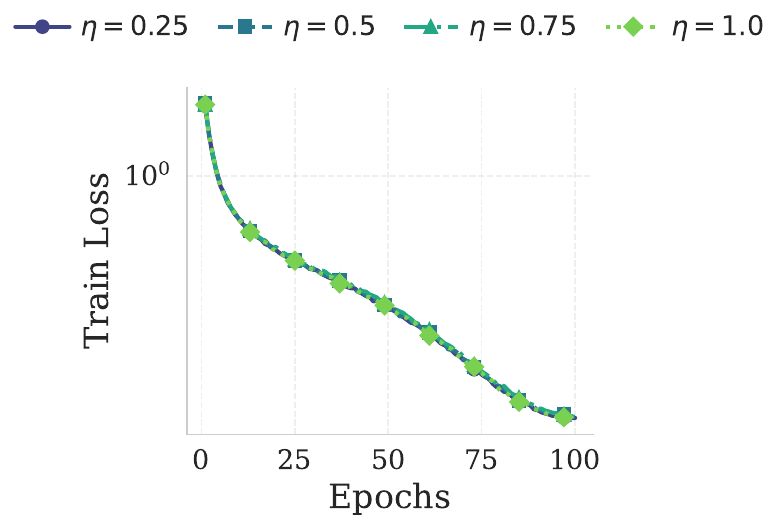}
  \caption{\textbf{Training Loss Consistency.} The decay of training loss for the relaxation algorithm matches exactly across all step sizes. Axes indicate Epochs (0--100) versus Loss (log scale).}
  \label{fig:trainingloss}
\end{figure}

\subsection{Gradient Norm Consistency}
\label{app:norm}
A critical requirement for replacing symbolic differentiation with physical relaxation is the preservation of gradient magnitude. We measured the norm ratio $\|\nabla_{\text{relax}}\| / \|\nabla_{\text{BP}}\|$ throughout training.

Figure \ref{fig:norm_ratio} demonstrates that the proposed method maintains near-perfect scale consistency. The deviation from unity is negligible, with fluctuations bounded within a minute range of $0.0001$ to $-0.0003$. This confirms that the stress variable $\bm{s} = \bm{x} - \bm{z}$ correctly accumulates the total sensitivity of the loss with respect to activations.

\begin{figure}[h]
  \centering
  \includegraphics[width=0.5\textwidth]{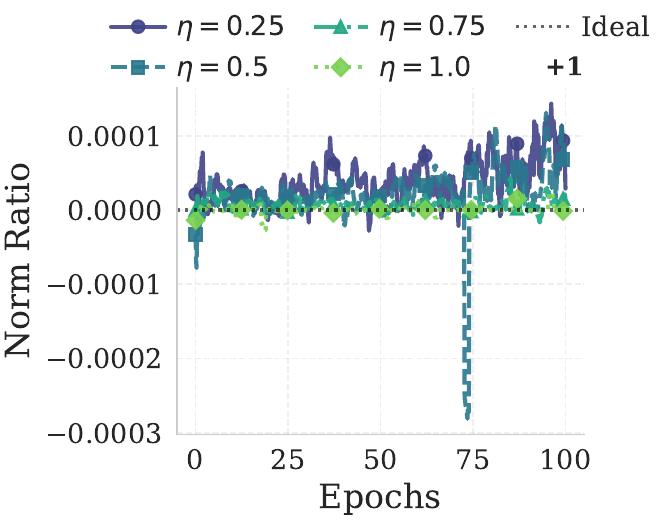}
  \caption{\textbf{Gradient Norm Fidelity.} The ratio between the norms of relaxation gradients and backpropagation gradients remains centered at 1.0, with deviations smaller than $10^{-4}$. Legend indicates step sizes $\eta=0.25$ through $\eta=1.0$.}
  \label{fig:norm_ratio}
\end{figure}

\subsection{Directional Alignment}
\label{app:dir_al}
We can also observe that the directional accuracy of the gradient update is paramount. We present here a global quantification of this using the global cosine misalignment metric, defined as $1 - \cos(\theta)$, where $\theta$ is the angle between $\nabla_{\text{relax}}$ and $\nabla_{\text{BP}}$ as measured in the main text.

Figure \ref{fig:cosine_mis} reveals that the misalignment is practically zero. The metric consistently falls below $10^{-5}$ oscillating around $10^{-7}$, confirming to be almost indistinguishable from standard backpropagation (the accuracy used in the code is IEEE 754 Float32single point precision corresponding to a decimal precision $\approx 10^{-7}$).

\begin{figure}[h]
  \centering
  \includegraphics[width=0.5\textwidth]{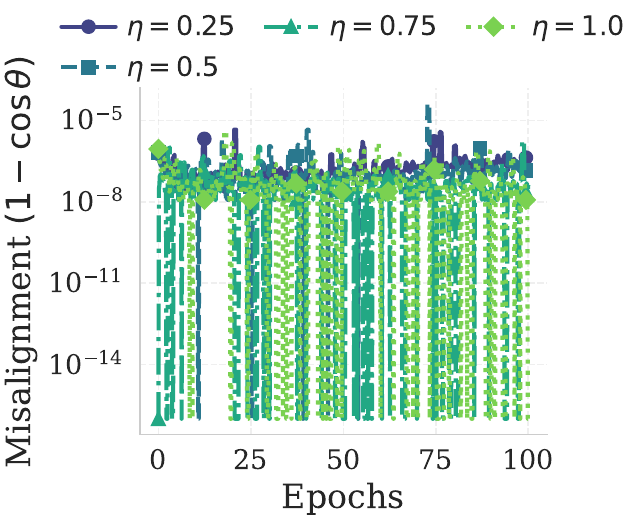}
  \caption{\textbf{Directional Alignment.} The cosine misalignment ($1-\cos \theta$) between the relaxation algorithm and standard BP.}
  \label{fig:cosine_mis}
\end{figure}

\subsection{Signal-to-Noise Ratio (SNR)}
Finally, we assessed the quality of the gradient estimator by computing the Signal-to-Noise Ratio (SNR), defined as $\|\nabla_{\text{BP}}\|^2 / \|\nabla_{\text{relax}} - \nabla_{\text{BP}}\|^2$.

As shown in Figure \ref{fig:snr}, the SNR remains exceptionally high throughout the training process, consistently exceeding $10^6$ (60 dB) and peaking near $10^9$. We observe a non-linear relationship where larger step sizes (e.g., $\eta=1.0$) yield much higher SNR.

\begin{figure}[h]
  \centering
  \includegraphics[width=0.5\textwidth]{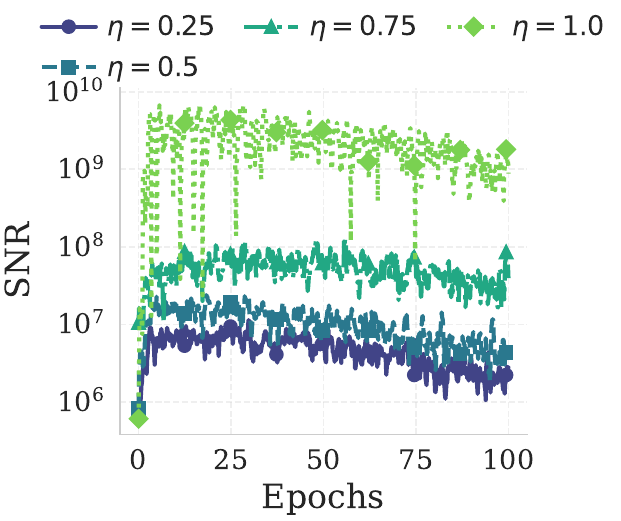}
  \caption{\textbf{Gradient Signal-to-Noise Ratio.} The SNR consistently exceeds $10^6$, confirming that the relaxation error is negligible compared to the gradient signal.}
  \label{fig:snr}
\end{figure}

\subsection{Layer-wise Misalignment and Time Step}
\label{app:layer_cos}
In our experimental analysis, we observed that while the overall gradient accuracy remains exceptionally high, a subtle degradation in alignment emerges as signals propagate backward to earlier layers (e.g., Layer 1) over the course of training. This phenomenon is visualized in Figure~\ref{fig:misalignment_eta025}, which displays the log-misalignment for a step size of $\eta=0.25$. This slight drift is attributable to the propagation of numerical errors inherent to the finite-time step approximation of the continuous dynamics.

Conversely, using the theoretically optimal step size $\eta=1.0$ (Figure~\ref{fig:misalignment_eta1}) recovers perfect alignment across all layers. We stress that digital implementation is not our goal; these experiments serve simply to demonstrate consistency and experimentally validate the theory.

\begin{figure}[h]
  \centering
  \includegraphics[width=0.5\textwidth]{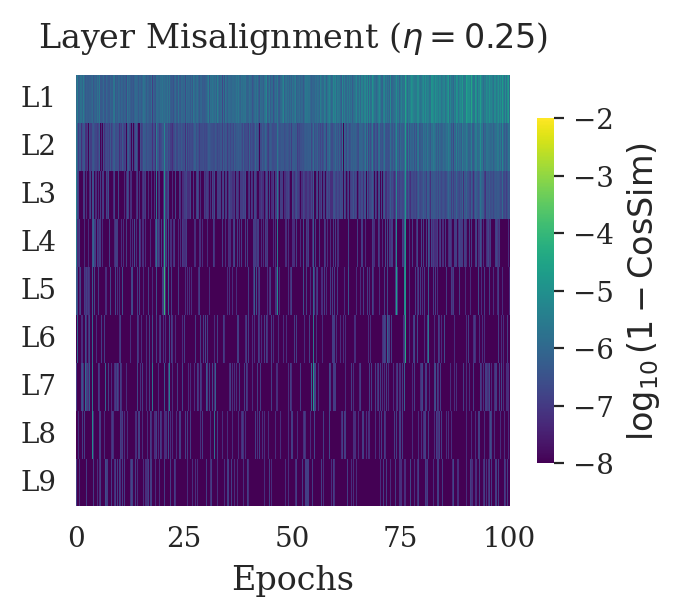}
  \caption{\textbf{Log-misalignment heatmap ($\eta=0.25$)}. The plot displays $\log_{10}(1 - \text{CosSim})$ for each layer. A minor accumulation of numerical error is visible in earlier layers, consistent with finite-step approximations.}
  \label{fig:misalignment_eta025}
\end{figure}

\begin{figure}[h]
  \centering
  \includegraphics[width=0.5\textwidth]{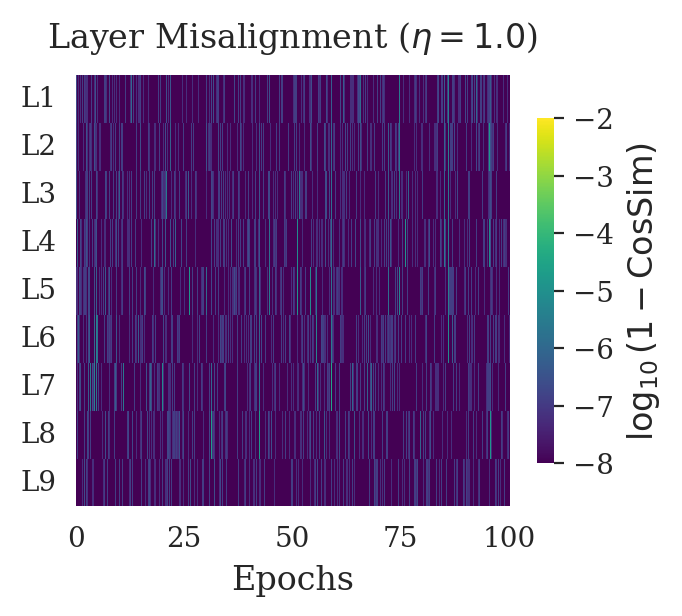}
  \caption{\textbf{Log-misalignment heatmap ($\eta=1.0$)}. The plot displays $\log_{10}(1 - \text{CosSim})$ for each layer. The alignment is effectively perfect (bounded only by machine precision), confirming that the unit-step discretization recovers the exact gradient.}
  \label{fig:misalignment_eta1}
\end{figure}

% NeurIPS requires the paper checklist. It is appended after the appendix and
% does not count toward the page limit.
% Omitted from the arXiv preprint version.
% \newpage
% \input{checklist.tex}

\end{document}